\documentclass{article} 
\usepackage{iclr2027_conference,times}

\newif\ifpreprint
\preprinttrue
\ifpreprint\iclrfinalcopy\fi

\usepackage{amsmath,amsfonts,bm}

\def\eqref#1{equation~\ref{#1}}

\def\1{\bm{1}}

\DeclareMathAlphabet{\mathsfit}{\encodingdefault}{\sfdefault}{m}{sl}
\SetMathAlphabet{\mathsfit}{bold}{\encodingdefault}{\sfdefault}{bx}{n}

\usepackage[utf8]{inputenc} 
\usepackage[T1]{fontenc}    
\usepackage{hyperref}       
\usepackage{url}            
\usepackage{booktabs}       
\usepackage{amsfonts}       
\usepackage{nicefrac}       
\usepackage{microtype}      
\usepackage[table]{xcolor}

\usepackage{graphicx}
\usepackage{amsmath}
\usepackage{amssymb}
\usepackage{algpseudocodex}

\usepackage{listings}
\usepackage{mwe} 
\usepackage{makecell}
\usepackage{color, colortbl}
\usepackage[normalem]{ulem} 
\usepackage{algorithm}

\usepackage{wrapfig}
\usepackage{caption}
\usepackage{pifont} 
\usepackage{multirow}
\usepackage{booktabs}
\usepackage{float}

\usepackage{chngpage}
\usepackage{wrapfig}
\usepackage{adjustbox}
\usepackage{comment}
\usepackage{color}

\usepackage{listings}

\usepackage{bm}
\usepackage[capitalize]{cleveref}

\usepackage[toc]{appendix}
\usepackage{etoc}
\usepackage{minitoc}
\usepackage{mathtools}

\usepackage{amsthm}

\newtheorem{theorem}{Theorem}[section]        

\theoremstyle{definition}

\theoremstyle{remark}

\usepackage{tikz}
\usetikzlibrary{
    arrows.meta,
    calc,
    positioning,
    decorations.pathreplacing,
    fit,
    shapes.geometric
}
\usepackage{placeins}
\usepackage{subcaption}
\crefname{section}{Sec.}{Secs.}
\Crefname{section}{Section}{Sections}
\Crefname{table}{Table}{Tables}
\crefname{table}{Tab.}{Tabs.}
\etocsettocstyle{\section*{Appendix}}{}

\definecolor{blsbg}{HTML}{EBF2E7}
\definecolor{cnpgreen}{RGB}{43, 122, 63}

\usepackage{cleveref}
\crefname{table}{Table}{Tables}
\usepackage[normalem]{ulem}

\newcommand{\at}[1]{{\tiny @#1}}

\newcommand{\pmstd}[1]{{\kern0.3pt\raisebox{0.1pt}{\scalebox{0.6}{\ensuremath{\,\pm\,}#1}}}}

\definecolor{Gray}{gray}{0.9}
\definecolor{ImportantColor}{rgb}{0.63, 0.79, 0.95}
\newcolumntype{g}{>{\columncolor{ImportantColor}}c}
\newcolumntype{?}{!{\vrule width 1pt}}
\graphicspath{
{figs/},
}

\newcommand*{\belowrulesepcolor}[1]{%
  \noalign{%
    \kern-\belowrulesep 
    \begingroup 
      \color{#1}%
      \hrule height\belowrulesep 
    \endgroup 
  }%
} 
\newcommand*{\aboverulesepcolor}[1]{%
  \noalign{%
    \begingroup 
      \color{#1}%
      \hrule height\aboverulesep 
    \endgroup 
    \kern-\aboverulesep 
  }%
}

\definecolor{codegreen}{rgb}{0,0.6,0}
\definecolor{codegray}{rgb}{0.5,0.5,0.5}
\definecolor{codepurple}{rgb}{0.58,0,0.82}
\definecolor{backcolour}{rgb}{0.95,0.95,0.92}

\lstdefinestyle{mystyle}{
    backgroundcolor=\color{backcolour},   
    commentstyle=\color{codegreen},
    keywordstyle=\color{magenta},
    numberstyle=\tiny\color{codegray},
    stringstyle=\color{codepurple},
    basicstyle=\ttfamily\footnotesize,
    breakatwhitespace=false,         
    breaklines=true,                 
    captionpos=b,                    
    keepspaces=true,                 
    numbers=left,                    
    numbersep=5pt,                  
    showspaces=false,                
    showstringspaces=false,
    showtabs=false,                  
    tabsize=2
}

\hypersetup{hidelinks}
\renewcommand{\bottomfraction}{0.8}
\renewcommand{\floatpagefraction}{0.8}

\newcommand{\restoreappendixlayout}{%
\setcounter{topnumber}{2}
\setcounter{totalnumber}{3}
\setlength{\textfloatsep}{20pt plus 2pt minus 4pt}
\setlength{\floatsep}{12pt plus 2pt minus 2pt}
\setlength{\intextsep}{12pt plus 2pt minus 2pt}
\renewcommand{\bottomfraction}{0.3}
\renewcommand{\floatpagefraction}{0.5}
\captionsetup{font=normalsize,skip=10pt}
}

\newcommand{\method}{BLS}

\title{Task Inference Beyond Least Squares in Behavioral Foundation Models}

\author{%
{
Kuan-Hsun Tu$^{1*}$,
Chien-Sheng Chiang$^{1\dagger *}$,
Hsin-Wei Chen$^{1}$,
Ping-Chun Hsieh$^{2}$,
and Tsung-Wei Ke$^{1}$
}\\[3pt]
{
$^{1}$National Taiwan University,
\quad
$^{2}$National Yang Ming Chiao Tung University
}
}

\begin{document}

\maketitle
\ifpreprint
\lhead{Preprint}
{\renewcommand{\thefootnote}{\fnsymbol{footnote}}\footnotetext[1]{Equal contribution. \quad $^{\dagger}$Corresponding author.}}
\fi

\begin{abstract}
Behavioral Foundation Models (BFMs) aim to solve a wide range of downstream tasks without test-time policy learning by inferring a task vector from the reward function. While efficient, the retrieved policies are often suboptimal because of how this task vector is inferred, typically with ordinary least
squares (OLS). OLS minimizes reward reconstruction error but leaves the ordering of rewards unconstrained, which can bias the successor measure of the retrieved zero-shot policy away from that of the optimal policy. In this work, we propose \method{}, an efficient test-time inference method that balances minimizing reward reconstruction error with reducing successor-measure mismatch. Theoretically, we provide a suboptimality gap upper bound characterized by both successor-measure and reward-function residuals. Empirically, we evaluate \method{} on top of state-of-the-art BFMs across benchmarks for locomotion, manipulation, and humanoid control. \method{} outperforms existing task inference baselines with negligible computational overhead. Project page: \href{https://embodiedai-ntu.github.io/BLS/index.html}{\textcolor{cnpgreen}{https://embodiedai-ntu.github.io/BLS}}
\end{abstract}

\begin{figure}[h]
    \centering
    \includegraphics[width=0.90\textwidth]{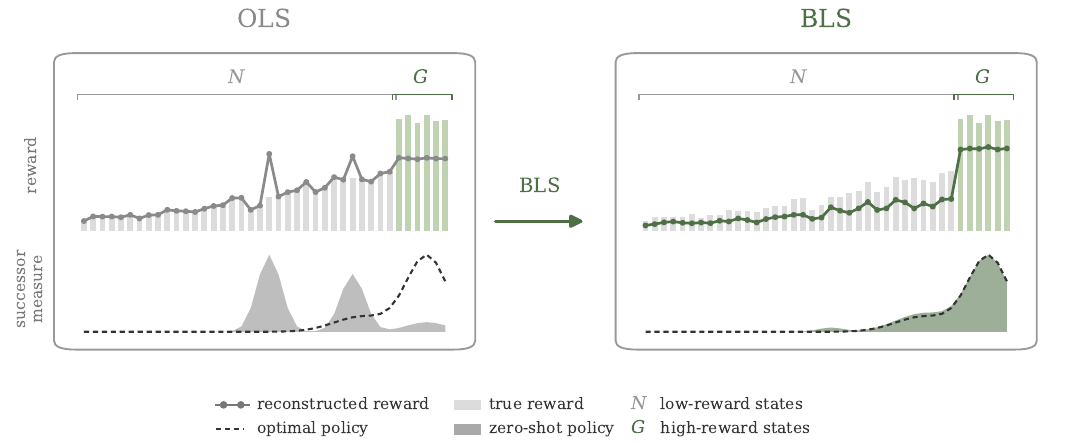}
    \caption{Ordinary least squares (OLS) aims to minimize the total errors of reconstructed rewards. Its performance highly depends on the quality of state representations, the state coverage within the dataset, and the distribution of rewards. Because OLS can be biased toward the majority of states, low-reward states ($N$) may be assigned with high-value rewards and vice versa for high-reward states ($G$). The successor measure of the retrieved zero-shot policy can therefore diverge from that of the optimal policy. \method{} instead adopts an objective that reduces the reward reconstruction errors while enlarging the margin between high- and low-reward states. Thus, the successor measure of the retrieved zero-shot policy is encouraged to stay closer to the optimal policy's.
    }
    \label{fig:teaser}
\end{figure}

\section{Introduction}

Unsupervised zero-shot reinforcement learning (zero-shot RL; \citealp{touati2022does}) aims to offer a more general behavior learning framework than typical task-specific RL.  During training, it does not require downstream task-specific reward functions, instead learning a repertoire of policies directly from interactions within the environment.  At test time, given a downstream reward function, it can immediately retrieve a corresponding policy without additional policy learning or planning. Models trained via zero-shot RL can jointly handle downstream tasks in diverse forms, such as tracking reference trajectories, reaching goals, optimizing rewards, or following semantic instructions~\citep{tirinzoni2025zero,sikchi2024rl}.  Therefore, these models are often termed \textit{Behavioral Foundation Models} (BFMs; \citealp{pirotta2024fast}).

To empower zero-shot capabilities, BFMs rely on a linear formulation that connects arbitrary tasks with their reward functions.  A BFM represents every state and \textit{all possible} tasks as vectors in a joint latent space, where their inner products define linear reward functions~\citep{barreto2017successor}. It then trains policies conditioned on each task vector using the associated reward. Given any reward function at test time, the BFM directly obtains the corresponding policy by inferring the task vector through linear regression, with ordinary least squares (OLS) being the most commonly used solver in prior work~\citep{bagatella2025td}.

Although OLS minimizes reconstruction errors of a reward function among all states, it does not guarantee that the retrieved zero-shot policy is the best-performing one for the reward. The performance of a zero-shot policy heavily depends on the quality of state representations, the coverage of states within the dataset, and the distribution of rewards per task~\citep{sikchi2025fast}.  Figure~\ref{fig:teaser} presents an example: when rewards are dense but the highest-reward states are sparse, OLS is biased toward the majority of states, which are low-reward states in this example.  This can leave the reconstruction of high-value rewards inaccurate, causing the retrieved zero-shot policy to become suboptimal. To tackle this, \citet{sikchi2025fast} propose a fast adaptation strategy that minimizes the residual errors of OLS reconstruction but needs task-specific online interactions.  \citet{asadulaev2026zero} instead search for a task vector that maximizes rewards, while constraining the retrieved zero-shot policy to follow trajectories within the dataset, which are not optimal for the downstream task.  Current methods remain restricted.

We propose \method{}, a simple yet effective offline zero-shot inference framework \underline{B}eyond \underline{L}east \underline{S}quares for BFMs. We first identify two sources of error that cause the suboptimality of zero-shot policies: drift in reward reconstruction and drift in state visitation distributions.  While the former is covered by OLS, the latter is broadly overlooked in existing approaches.  We then make two key observations of the optimal policy's state visitation distributions: (1) high-reward states are more frequently visited than low-reward states, (2) fixing the relative ranking of rewards also helps to preserve the optimal policy's behavior.  To achieve both objectives, we propose a soft-margin contrastive loss that enlarges the margin in reconstructed rewards between high- and low-value original rewards. Finally, \method{} jointly optimizes the task vector using the soft-margin loss and a trust-region loss with respect to OLS's solution.  With the updated task vector, \method{} aims to retrieve a zero-shot policy closer to the optimal policy than the one found by OLS.

We validate \method{} with extensive experiments across several established RL benchmarks. On OGBench, \method{} improves the average performance of all five feature representations under both proprioceptive and pixel-based observations. On DMC and HumEnv, \method{} improves 18 out of 20 and 42 out of 45 dense-reward tasks, respectively. In particular, we find that \method{} remains effective across different feature representations, reward structures, and task vector anchors. Furthermore, \method{} incurs negligible computational overhead compared to OLS while bringing substantial performance gains. Finally, these results show that \method{} is not merely an OLS correction but a general task-vector refinement framework for BFMs.


\section{Related Work}
\paragraph{Behavioral Foundation Models.}

Behavioral Foundation Models (BFMs) build on successor features~\cite{dayan1993improving, barreto2017successor} to learn a family of policies. Early approaches typically rely on separately learned state representations~\cite{wu2018laplacian, ghosh2023reinforcement, park2024foundation} and approximate the successor features induced by these fixed representations. Later, \citet{touati2021learning, blier2021learning, touati2022does} introduce forward-backward representations that jointly learn the state features and successor measure. This approach achieves state-of-the-art performance on classical locomotion and manipulation tasks and has also been extended to humanoid control~\cite{tirinzoni2025zero, li2026bfm}. Recent work further explores applications including zero-shot imitation learning~\cite{pirotta2024fast}, general utility functions~\cite{bagatella2026soft}, and preference-based RL~\cite{yang2026reward}. To better model successor measures, TD-JEPA~\cite{bagatella2025td} learns latent predictive representations that capture long-term dynamics, while \citet{farebrother2025temporal} uses flow matching to approximate the successor measure. Rather than modifying how representations and successor measures are trained, we optimize the inferred task vector $z$ while keeping the representation, successor measure, and policy fixed.

\paragraph{Inferring the task vector.}
In prior work on BFMs, ordinary least squares (OLS) is the most commonly used method to infer the task vector. To mitigate the reward projection error, ReLA~\cite{sikchi2025fast} proposes an off-policy actor-critic method that trains a small residual critic to compensate for the reward projection error while updating only the task vector $z$. LoLA~\cite{sikchi2025fast} is an actor-only alternative that optimizes a distribution over $z$ from fixed-horizon on-policy rollouts bootstrapped by the frozen pretrained critic. Other works instead study task inference under different settings. OpTI-BFM~\cite{rupf2026optimistic} considers a practical setting in which reward labels are obtained only through environment interaction. It maintains uncertainty over the reward function from observed feature--reward pairs and selects $z$ optimistically within the resulting confidence set, casting task inference as a linear bandit problem. ZOL~\cite{asadulaev2026zero} connects successor measures to stationary density ratios, showing that forward--backward representations implicitly encode the importance sampling weights between the policy- and data-induced occupancies. It uses these weights to perform distribution correction on the offline task inference objective, without additional training or interaction.

These works either operate within the learned feature basis or use online interaction to compensate for residual error. In contrast, \method{} performs offline test-time task-vector refinement directly from reward-feature samples, using a surrogate objective motivated by our suboptimality gap.
\section{Preliminaries}
We consider a reward-free Markov Decision Process (MDP) $\mathcal{M} = (\mathcal{S}, \mathcal{A}, P, \gamma, \rho)$, where $\mathcal{S}$ and $\mathcal{A}$ denote the state and action spaces, $P(\cdot \mid s, a)$ is the transition kernel, $\gamma \in (0,1)$ is the discount factor, and $\rho$ is the initial state distribution. A downstream task is specified by a bounded reward function $r : \mathcal{S} \to \mathbb{R}$. An agent interacts with the MDP according to a policy $\pi:\mathcal{S}\rightarrow \mathrm{Prob}(\mathcal{A})$. Let $V_r^\pi = \mathbb{E}_{s_0 \sim \rho, a_t \sim \pi(\cdot \mid s_t), s_{t+1} \sim P(\cdot \mid s_t, a_t)}\left[ \sum_{t=0}^{\infty} \gamma^{t} r(s_{t+1}) \right]$ denote the expected cumulative discounted return, or \textit{value}, of a policy $\pi$ under a reward function $r$. RL aims to find an optimal policy $\pi_r$ that maximizes the value: $\pi_r = \arg\max_{\pi} V_r^\pi$.

\paragraph{Successor measure.}
Executing a policy $\pi$ in the MDP induces stochastic trajectories, which can be described by the discounted visitation measure over future states, also termed the \textit{successor measure}~\cite{dayan1993improving}.  Formally, it is written as: $ M^\pi(s^+ \mid s,a)=\sum_{t=0}^{\infty}\gamma^t\mathrm{Pr}\left(s_{t+1}=s^+ \mid s, a; \pi\right)$, 
where $\mathrm{Pr}\left(s_{t+1}\mid s, a;\pi\right)$ is the probability of visiting state $s_{t+1}$ at time step $t+1$ starting from initial state-action pair $(s, a)$, following the environment dynamics $s_{t+1} \sim P(\cdot | s_{t}, a_{t})$ and the policy $a_t \sim \pi(\cdot | s_t)$.  
With the successor measure, the value $V_r^\pi$ and the state-action value function $Q_r^\pi: \mathcal{S} \times \mathcal{A} \to \mathbb{R}$ of a policy $\pi$ under a reward function $r$ can be written as $V_r^\pi=\mathbb{E}_{s_0 \sim \rho, a_0 \sim \pi(\cdot \mid s_0)} \left[ \sum_{s^+ \in \mathcal{S}} M^\pi(s^+ \mid s_0, a_0) r(s^+) \right]$ and $Q_r^\pi(s, a)=\sum_{s^+ \in \mathcal{S}} M^\pi(s^+ \mid s, a) r(s^+)$.

\paragraph{Zero-shot RL and behavior foundation models.}
Unlike typical RL that learns an optimal policy based on a reward function specifying a task, zero-shot RL seeks to optimize policies for \textit{all possible} tasks without any reward function given during training.  To eliminate the need for handcrafting rewards, zero-shot RL adopts a linear form that associates every reward function with state features~\citep{touati2021learning}: $r_z(s)=\phi(s)^\top z$, where $\phi : \mathcal{S} \to \mathbb{R}^d$ is the state feature map, and $z \in \mathbb{R}^d$ denotes the task vector characterized by the reward function $r_z$.  This formulation automatically constructs reward functions representing arbitrary tasks with randomly sampled task vectors.  To address challenges in learning policies for every task, zero-shot RL combines the linear formulation with the successor measure, resulting in a self-supervised learning scheme~\citep{touati2022does}:
\begin{equation}
\begin{cases}
M^{\pi_z}(\cdot \mid s,a)
=
P(\cdot \mid s,a)
+
\gamma
\mathbb{E}_{s' \sim P(\cdot \mid s,a),\, a' \sim \pi_z(\cdot \mid s')}
\left[
M^{\pi_z}(\cdot \mid s',a')
\right], \\

Q_{r_z}^{\pi_z}(s,a)
=
\sum_{s^+ \in \mathcal{S}}
M^{\pi_z}(s^+ \mid s,a)\, r_z(s^+)
=
\sum_{s^+ \in \mathcal{S}}
M^{\pi_z}(s^+ \mid s,a)\, \phi(s^+)^\top z, \\

\pi_z(s)
=
\underset{a}{\mathrm{argmax}}\;
Q_{r_z}^{\pi_z}(s,a).
\end{cases}
\end{equation}
where $\pi_z$ denotes the policy conditioned on the task vector $z$.  In other words, for each $z$ and the policy $\pi_z$, zero-shot RL trains to model the policy's successor measure $M^{\pi_z}$ with the Bellman equation, while optimizing the policy to maximize state-action values $Q_{r_z}^{\pi_z}$ estimated with the successor measure and the reward $r_z$.  At inference, given any reward function $r$, the task vector is calculated via OLS:
\begin{align}
\hat{z}_{\mathrm{OLS}} = \underset{z}{\mathrm{argmin}} \mathbb{E}_s \left[ \lVert \phi(s)^\top z - r(s) \rVert_2^2 \right].
\label{eq:ols}
\end{align}
With $\hat{z}_{\mathrm{OLS}}$, we retrieve the policy $\pi_{\hat{z}_{\mathrm{OLS}}}$ for the task represented by the reward function.  Because we can directly obtain policies for any task (reward functions) with no extra training or planning, these models are often referred to as BFMs. Notably, if the reward function $r$ lies in the span of the learned features $\phi$, OLS exactly reconstructs $r$, and the retrieved zero-shot policy $\pi_{\hat{z}_{\mathrm{OLS}}}$ is optimal for the task~\citep{ollivier2025features}. However, arbitrary downstream rewards may not lie in this feature span in practice, leading to non-negligible reconstruction errors and degraded task performance.


\section{Beyond Least Squares}
BLS is an effective zero-shot inference framework for BFMs, aiming to retrieve a better-performing policy than OLS for any reward function given at test time. We first analyze the value discrepancy, or \textit{suboptimality gap}, between the optimal policy and a policy corresponding to a reconstructed reward function. We show that the suboptimality gap depends on both the reward reconstruction error and the difference between the two policies' successor measures, each of which motivates an objective for zero-shot inference. While the former coincides with OLS (Eqn.~\ref{eq:ols}), the latter is intractable without knowing the optimal policy's successor measure in advance.

To bridge the gap, we gain insights from inductive biases of successor measures: 
optimal policies visit high-reward states more frequently than low-reward states, and reward functions sharing similar reward rankings tend to induce similar successor measures\footnote{This is an empirical observation rather than a theoretical guarantee, though we find it holds in most practical settings.}.  We propose a soft-margin contrastive loss to improve the reconstruction of high-value rewards and preserve the relative ranking of rewards. The task vector is optimized with both objectives via gradient descent.

We present the suboptimality gap analysis in Section~\ref{subsec:mtd_theorem} and the optimization objectives in Section~\ref{subsec:mtd_algorithm}.
\subsection{Theoretical Analysis}
\label{subsec:mtd_theorem}
We start by analyzing the suboptimality gap between the optimal policy and a policy retrieved by zero-shot inference.  Let $r$ and $\hat{r}$ denote the original and reconstructed reward functions.  $\pi_r$ and $\pi_{\hat{r}}$ are the optimal policies for $r$ and $\hat{r}$, respectively.  The suboptimality gap between these two policies is defined as $\mathrm{SubOpt}(r, \hat{r}) \coloneq V_r^{\pi_r} - V_r^{\pi_{\hat{r}}}$. Next, we show that the upper bound on the suboptimality gap depends on both the reward reconstruction errors and the difference between the successor measures of $\pi_r$ and $\pi_{\hat{r}}$.

\begin{theorem}[Suboptimality gap]
\label{thm:main}
Let $\pi_r$ and $\pi_{\hat{r}}$ be the optimal policies for rewards $r$ and $\hat{r}$, respectively. Define the policy-averaged successor measure as $\bar{M}^{\pi}(s^+ \mid s) \coloneq \mathbb{E}_{a \sim \pi(\cdot \mid s)}[M^{\pi}(s^+ \mid s,a)]$. We define $\Delta M(s^+ \mid s)\coloneq\bar{M}^{\pi_{\hat{r}}}(s^+ \mid s) - \bar{M}^{\pi_r}(s^+ \mid s)$ as the residual between the two successor measures, and $\Delta r(s) \coloneq \hat{r}(s) - r(s)$ as the residual between the two rewards. Then, the suboptimality gap $\mathrm{SubOpt}(r, \hat{r})$ of $\hat{r}$ has an upper bound: 
\[
  \mathrm{SubOpt}(r, \hat{r})
  \;\le\;
  \lVert \Delta M \rVert_{1} \;\cdot\; \lVert \Delta r \rVert_{\infty} .
\]
\end{theorem}
 
The proof is in Appendix~\ref{app:proof}.  Theorem~\ref{thm:main} reveals two conditions under which $\pi_{\hat{r}}$ achieves the same value as $\pi_r$ under $r$. The first occurs when the reconstructed reward function $\hat{r}$ is the same as the original reward function $r$, and the second holds when $\pi_{\hat{r}}$'s policy-averaged successor measure matches that of $\pi_r$, even if the reconstructed rewards are incorrect. Notably, OLS pursues only the first condition while ignoring the second one.  Because OLS in practice cannot reconstruct exact rewards, the residual of rewards can change the zero-shot policy's behavior and enlarge the residual of successor measures.  The policy can therefore be suboptimal.  In contrast, our key insight is to consider these conditions jointly, seeking the task vector that minimizes both the residual of rewards and successor measures.

\subsection{Practical algorithm}
\label{subsec:mtd_algorithm}
Devising an objective that explicitly minimizes the residual of successor measures remains a substantial challenge, as the optimal policy's successor measure is typically unknown at test time.  We therefore investigate a surrogate objective based on inductive biases.  We observe that high-density regions of a successor measure often correspond to high-reward states, while lowering (raising) rewards at high-(low-) reward states may significantly change the policy's behavior and also its successor measure.  Based on these observations, we conjecture that preserving both high-value rewards and relative ranking of rewards is an effective and efficient objective to achieve the optimal successor measure.  This motivates a soft-margin contrastive loss.

We propose an iterative optimization algorithm.  Starting with OLS to find a task vector, the algorithm optimizes the vector via gradient descent, using a soft-margin contrastive loss and a trust-region loss.  The former aims to reduce the residual of successor measures, and the latter constrains the task vector to stay close to OLS's solution, which minimizes the total reward reconstruction errors. We summarize the algorithm in Algorithm~\ref{alg:bls}.
 
\paragraph{Soft-margin contrastive loss.}
We separate all states into two sets: high-reward states $G$ and low-reward states $N$, with a threshold $\tau$.  The core idea is to enlarge the margins of reconstructed rewards between these two sets.  However, a hard-margin contrastive loss $\mathcal{L}(z) = \max_{s \in N} \phi(s)^\top z - \min_{s \in G} \phi(s)^\top z$ is limited.  It considers only one margin between the lowest reconstructed reward from $G$ and the highest from $N$, ignoring other violating margins.  Moreover, the loss contains non-smooth operations and is therefore less suitable for gradient-based optimization. To address this, we propose a soft-margin contrastive loss:
\begin{align}
    \mathcal{L}_{sm}(z) = \log \sum_{s \in G} \exp\!\big(\!- \phi(s)^{\top} z\big) +  \log \sum_{s \in N} \exp\!\big( \phi(s)^{\top} z\big).
    \label{eq:softmargin_loss}
\end{align}
Intuitively, the first term of $\mathcal{L}_{sm}$ pushes up the bottom reconstructed rewards from $G$, while the second term drives down the top ones from $N$. The relative ranking among rewards is encouraged.

\paragraph{Trust-region loss.}
The soft-margin contrastive loss alone does not explicitly control the reward reconstruction error. We additionally introduce a trust-region loss: it regularizes the task vector $z$ to stay close to OLS's solution, which minimizes the total reconstruction errors.  The trust-region loss is written as $\mathcal{R}(z)=\mathbb{E}_s \left[ \big\lVert \phi(s)^\top z - \phi(s)^\top \hat{z}_{\mathrm{OLS}} \big\rVert_2^2 \right]$.  Notably, this loss penalizes the discrepancy between rewards reconstructed from $z$ and $ \hat{z}_{\mathrm{OLS}}$, rather than the difference between the task vectors themselves.  This design is better aligned with the objective of reducing the residual of rewards.

Finally, the total loss combines the soft-margin contrastive loss and the trust-region loss: $\mathcal{L}(z)=\mathcal{L}_{sm}(z) + \lambda \mathcal{R}(z)$ with a coefficient $\lambda$ balancing these two losses.  In our experiments, we show that the choice of $\lambda$ correlates with the type of reward function: the trust-region loss (or OLS) benefits dense rewards, while the soft-margin contrastive loss is more helpful to sparse rewards. We further analyze the effect of $\lambda$ in Section~\ref{subsec:ablation}.

\begin{algorithm}[hb]
\small
\caption{BLS}
\label{alg:bls}
\begin{algorithmic}[1]
\Require a reward function $r$, a state feature map $\phi$, threshold $\tau$, number of iterations $T$, coefficient $\lambda$, and learning rate $\alpha$.
\Ensure $z$
\State $G \gets \{s : r(s) \geq \tau\}$
\State $N \gets \{s : r(s) < \tau\}$
\State $\hat{z}_{\mathrm{OLS}} \gets \arg\min_{z} \mathbb{E}_{s} \left[ \lVert \phi(s)^\top z - r(s) \rVert_2^2 \right]$
\State $z \gets \hat{z}_{\mathrm{OLS}}$
\For{$t = 1 \ldots T$}
    \State $\mathcal{L}_{sm}(z) \gets
    \log \sum_{s \in G} \exp\left(-\phi(s)^\top z\right)
    + \log \sum_{s \in N} \exp\left(\phi(s)^\top z\right)$
    \State $\mathcal{R}(z) \gets
    \mathbb{E}_{s} \left[
    \left\lVert
    \phi(s)^\top z -
    \phi(s)^\top \hat{z}_{\mathrm{OLS}}
    \right\rVert_2^2
    \right]$
    \State $z \gets z - \alpha \nabla_z
    \left(
    \mathcal{L}_{sm}(z) + \lambda \mathcal{R}(z)
    \right)$
    \State $z \gets z / \lVert z \rVert_2$
\EndFor
\State \Return $z$
\end{algorithmic}
\end{algorithm}

\begin{table}[t]
\caption{Zero-shot success rates (\%) on OGBench across five feature representations under proprioceptive and pixel-based observations. BLS improves the average performance of all representations.}
\label{tab:ogbench}
\footnotesize
\setlength{\tabcolsep}{2.5pt}
\renewcommand{\arraystretch}{1.25}

\resizebox{5.5in}{!}{%
\begin{tabular}{@{} l *{5}{cc} @{}}
\toprule
 & \multicolumn{2}{c}{Laplacian}
 & \multicolumn{2}{c}{ICVF}
 & \multicolumn{2}{c}{HILP}
 & \multicolumn{2}{c}{FB}
 & \multicolumn{2}{c}{TD-JEPA} \\
\cmidrule(lr){2-3} \cmidrule(lr){4-5} \cmidrule(lr){6-7} \cmidrule(lr){8-9} \cmidrule(lr){10-11}
 & {OLS} & {BLS}
 & {OLS} & {BLS}
 & {OLS} & {BLS}
 & {OLS} & {BLS}
 & {OLS} & {BLS} \\
\midrule
antmaze-mn & 48.1\pmstd{2.9} & \textbf{67.3}\pmstd{8.1} & 66.3\pmstd{1.2} & \textbf{69.2}\pmstd{6.7} & 74.3\pmstd{1.4} & \textbf{87.7}\pmstd{1.5} & 77.5\pmstd{2.6} & \textbf{88.8}\pmstd{4.2} & 63.1\pmstd{0.8} & \textbf{83.2}\pmstd{7.7} \\
antmaze-ln & 7.6\pmstd{1.8} & \textbf{10.9}\pmstd{2.9} & 56.0\pmstd{4.1} & \textbf{65.5}\pmstd{4.9} & 49.2\pmstd{2.1} & \textbf{63.7}\pmstd{1.8} & 52.4\pmstd{4.0} & \textbf{63.5}\pmstd{9.2} & 50.7\pmstd{2.4} & \textbf{64.3}\pmstd{7.0} \\
antmaze-ms & 15.1\pmstd{7.8} & \textbf{34.7}\pmstd{3.4} & 28.8\pmstd{5.6} & \textbf{36.5}\pmstd{9.1} & 26.3\pmstd{9.1} & \textbf{31.3}\pmstd{2.3} & 69.9\pmstd{1.7} & \textbf{79.6}\pmstd{2.1} & 54.5\pmstd{4.8} & \textbf{59.5}\pmstd{2.4} \\
antmaze-ls & 15.2\pmstd{3.9} & \textbf{15.3}\pmstd{0.6} & 5.5\pmstd{0.8} & 4.0\pmstd{0.7} & 11.9\pmstd{1.2} & \textbf{13.5}\pmstd{0.8} & 51.5\pmstd{2.9} & 45.7\pmstd{4.6} & 39.2\pmstd{5.9} & \textbf{45.1}\pmstd{5.9} \\
antmaze-me & 3.7\pmstd{3.3} & \textbf{7.6}\pmstd{1.7} & 0.3\pmstd{0.5} & 0.0\pmstd{0.0} & 2.4\pmstd{2.6} & \textbf{6.7}\pmstd{1.5} & 50.3\pmstd{2.0} & 49.2\pmstd{4.9} & 15.1\pmstd{5.5} & \textbf{29.2}\pmstd{7.0} \\
cube-single & 19.9\pmstd{1.3} & 18.3\pmstd{4.6} & 19.5\pmstd{1.2} & 13.3\pmstd{0.5} & 84.5\pmstd{5.7} & 82.0\pmstd{6.7} & 50.4\pmstd{0.8} & \textbf{54.0}\pmstd{1.8} & 20.3\pmstd{2.9} & \textbf{55.9}\pmstd{3.0} \\
cube-double & 1.3\pmstd{0.6} & \textbf{2.0}\pmstd{1.8} & 3.1\pmstd{1.5} & \textbf{3.5}\pmstd{0.8} & 2.8\pmstd{3.1} & \textbf{25.6}\pmstd{2.1} & 1.9\pmstd{1.0} & \textbf{3.1}\pmstd{1.2} & 1.7\pmstd{0.5} & \textbf{2.1}\pmstd{1.2} \\
scene & 4.3\pmstd{1.4} & 4.0\pmstd{1.7} & 39.3\pmstd{3.2} & \textbf{47.7}\pmstd{3.7} & 26.3\pmstd{2.7} & \textbf{36.8}\pmstd{2.2} & 7.6\pmstd{0.7} & \textbf{9.6}\pmstd{0.4} & 37.9\pmstd{3.2} & \textbf{54.0}\pmstd{6.8} \\
puzzle-3$\times$3 & 1.5\pmstd{0.6} & \textbf{1.6}\pmstd{0.8} & 13.6\pmstd{0.7} & \textbf{15.3}\pmstd{1.4} & 2.4\pmstd{0.0} & \textbf{4.5}\pmstd{1.5} & 5.7\pmstd{1.2} & \textbf{9.6}\pmstd{1.8} & 11.2\pmstd{0.4} & \textbf{14.1}\pmstd{2.1} \\
\midrule[\heavyrulewidth]
Average & 13.0\pmstd{1.1} & \textbf{18.0}\pmstd{1.2} & 25.8\pmstd{0.9} & \textbf{28.3}\pmstd{1.4} & 31.1\pmstd{1.4} & \textbf{39.1}\pmstd{0.9} & 40.8\pmstd{0.7} & \textbf{44.8}\pmstd{1.4} & 32.6\pmstd{1.2} & \textbf{45.3}\pmstd{1.8} \\
\midrule
vis-antmaze-mn & 96.7\pmstd{0.8} & 95.5\pmstd{3.5} & 85.9\pmstd{1.5} & \textbf{89.3}\pmstd{1.6} & 82.3\pmstd{3.7} & \textbf{93.3}\pmstd{0.6} & 96.7\pmstd{0.8} & 95.3\pmstd{1.3} & 75.5\pmstd{3.8} & \textbf{88.8}\pmstd{2.4} \\
vis-antmaze-ln & 14.7\pmstd{5.7} & \textbf{20.7}\pmstd{1.2} & 45.6\pmstd{2.0} & 42.0\pmstd{4.2} & 65.9\pmstd{3.5} & \textbf{79.2}\pmstd{5.6} & 80.0\pmstd{2.5} & 77.9\pmstd{1.6} & 61.6\pmstd{2.6} & \textbf{66.9}\pmstd{6.1} \\
vis-antmaze-ms & 46.5\pmstd{1.5} & \textbf{66.3}\pmstd{1.2} & 31.5\pmstd{0.8} & \textbf{32.0}\pmstd{5.6} & 62.9\pmstd{2.8} & \textbf{63.5}\pmstd{1.7} & 80.1\pmstd{0.8} & 80.0\pmstd{1.1} & 74.9\pmstd{2.4} & \textbf{76.1}\pmstd{5.5} \\
vis-antmaze-ls & 1.5\pmstd{1.3} & \textbf{2.3}\pmstd{2.2} & 6.9\pmstd{3.8} & \textbf{7.3}\pmstd{1.1} & 9.7\pmstd{4.8} & \textbf{16.4}\pmstd{5.6} & 10.5\pmstd{4.4} & 10.0\pmstd{6.6} & 31.3\pmstd{2.1} & \textbf{34.1}\pmstd{0.8} \\
vis-antmaze-me & 0.0\pmstd{0.0} & 0.0\pmstd{0.0} & 0.0\pmstd{0.0} & 0.0\pmstd{0.0} & 0.0\pmstd{0.0} & 0.0\pmstd{0.0} & 0.1\pmstd{0.2} & 0.0\pmstd{0.0} & 0.0\pmstd{0.0} & 0.0\pmstd{0.0} \\
vis-cube-single & 45.7\pmstd{9.8} & \textbf{50.9}\pmstd{2.1} & 35.9\pmstd{3.6} & \textbf{36.0}\pmstd{3.7} & 66.3\pmstd{2.0} & \textbf{68.7}\pmstd{3.0} & 47.9\pmstd{2.7} & \textbf{56.7}\pmstd{4.6} & 47.7\pmstd{3.6} & \textbf{52.4}\pmstd{4.3} \\
vis-cube-double & 0.1\pmstd{0.2} & \textbf{3.3}\pmstd{5.8} & 0.0\pmstd{0.0} & 0.0\pmstd{0.0} & 2.5\pmstd{0.6} & 2.5\pmstd{1.7} & 0.9\pmstd{0.5} & 0.5\pmstd{0.5} & 0.5\pmstd{0.5} & \textbf{3.5}\pmstd{1.6} \\
vis-scene & 2.8\pmstd{0.7} & \textbf{16.3}\pmstd{9.9} & 2.8\pmstd{1.4} & \textbf{3.1}\pmstd{0.5} & 8.7\pmstd{1.7} & \textbf{28.5}\pmstd{4.3} & 14.5\pmstd{2.2} & \textbf{25.3}\pmstd{3.8} & 28.9\pmstd{2.7} & \textbf{34.9}\pmstd{2.8} \\
vis-puzzle-3$\times$3 & 12.4\pmstd{3.0} & \textbf{17.1}\pmstd{1.5} & 0.0\pmstd{0.0} & 0.0\pmstd{0.0} & 2.4\pmstd{0.4} & \textbf{2.5}\pmstd{1.2} & 7.2\pmstd{1.1} & 6.3\pmstd{2.3} & 0.8\pmstd{0.4} & \textbf{1.6}\pmstd{1.4} \\
\midrule[\heavyrulewidth]
Average & 24.5\pmstd{1.3} & \textbf{30.3}\pmstd{1.4} & 23.2\pmstd{0.7} & \textbf{23.3}\pmstd{0.9} & 33.4\pmstd{0.9} & \textbf{39.4}\pmstd{1.1} & 37.5\pmstd{0.7} & \textbf{39.1}\pmstd{1.1} & 35.7\pmstd{0.8} & \textbf{39.8}\pmstd{1.1} \\
\bottomrule
\end{tabular}
}
\end{table}
\vspace{5pt}
\section{Experiments}
\label{sec:experiments}
We systematically evaluate \method{} on established RL simulation benchmarks across a wide range of tasks, including 5 navigation and 4 manipulation tasks on OGBench~\citep{ICLR2025_ecd92623}, 20 continuous-control tasks across 4 domains on ExoRL/DMC~\citep{tassa2018deepmind, yarats2022don}, and 45 motion-tracking tasks with a humanoid robot on HumEnv~\citep{tirinzoni2025zero}. We evaluate \method{} along five main aspects. First, we examine its robustness to different feature representations. Second, we compare \method{} against existing zero-shot inference methods, including two offline approaches (OLS and ZOL~\citep{asadulaev2026zero}) and two online adaptation methods (ReLA and LoLA~\citep{sikchi2025fast}). Third, we compare \method{} with goal-conditioned RL. Fourth, we evaluate \method{} separately on dense- and sparse-reward tasks. Finally, we conduct ablations and further analyses to examine its design choices, robustness, and computational efficiency. Unless otherwise stated, we report the mean and standard deviation over 3 random seeds. Additional experimental results and analyses are provided in Appendix~\ref{app:additional_results}.

\subsection{Is \method{} robust to different choices of feature representations?}

We evaluate \method{} on OGBench across 5 feature representations: Laplacian~\citep{wu2018laplacian}, ICVF~\citep{ghosh2023reinforcement}, HILP~\citep{park2024foundation}, FB~\citep{touati2021learning}, and TD-JEPA~\citep{bagatella2025td}. These representations are pretrained with different objectives. OGBench contains 9 sparse-reward navigation and manipulation tasks. We consider both proprioceptive and pixel-based observations. For each representation and observation modality, we use a fixed $\lambda$ across all tasks; full configuration details are provided in Appendix~\ref{app:implementation}.

As shown in Table~\ref{tab:ogbench}, \method{} improves the average performance of all feature representations under both observation modalities. Under proprioceptive observations, the gains are particularly large for HILP ($31.1\% \rightarrow 39.1\%$) and TD-JEPA ($32.6\% \rightarrow 45.3\%$). The same overall trend holds for pixel-based observations, showing that \method{} is robust across both feature representations and observation settings.

\subsection{How does \method{} compare with existing task inference methods?}
\begin{table*}[t]
\centering

\begin{minipage}[t]{0.485\textwidth}
\centering
\captionof{table}{Average results over FB on the OGBench comparing BLS against offline baseline (ZOL) from different task anchor (OLS/ZOL).}
\label{tab:baselines_offline}

\footnotesize
\setlength{\tabcolsep}{2.8pt}
\renewcommand{\arraystretch}{1.12}

\begin{tabular*}{\linewidth}{@{\extracolsep{\fill}}lcccc@{}}
\toprule
& OLS
& ZOL
& BLS\at{OLS}
& BLS\at{ZOL} \\
\midrule

antmaze-mn
& 77.5\pmstd{2.6}
& 84.5\pmstd{1.6}
& \textbf{88.8}\pmstd{4.2}
& 84.4\pmstd{4.5} \\

antmaze-ln
& 52.4\pmstd{4.0}
& 41.5\pmstd{0.8}
& 63.5\pmstd{9.2}
& \textbf{68.7}\pmstd{2.4} \\

antmaze-ms
& 69.9\pmstd{1.7}
& 51.9\pmstd{1.2}
& \textbf{79.6}\pmstd{2.1}
& 68.1\pmstd{3.6} \\

antmaze-ls
& \textbf{51.5}\pmstd{2.9}
& 37.9\pmstd{5.1}
& 45.7\pmstd{4.6}
& 36.0\pmstd{4.6} \\

antmaze-me
& \textbf{50.3}\pmstd{2.0}
& 49.7\pmstd{3.1}
& 49.2\pmstd{4.9}
& 48.8\pmstd{1.4} \\

cube-single
& 50.4\pmstd{0.8}
& 46.4\pmstd{6.8}
& \textbf{54.0}\pmstd{1.8}
& 43.7\pmstd{6.5} \\

cube-double
& 1.9\pmstd{1.0}
& 2.8\pmstd{0.0}
& \textbf{3.1}\pmstd{1.2}
& 1.7\pmstd{1.0} \\

scene
& 7.6\pmstd{0.7}
& 33.5\pmstd{7.9}
& 9.6\pmstd{0.4}
& \textbf{35.3}\pmstd{8.4} \\

puzzle-3$\times$3
& 5.7\pmstd{1.2}
& \textbf{13.7}\pmstd{1.2}
& 9.6\pmstd{1.8}
& 11.7\pmstd{1.6} \\

\midrule[\heavyrulewidth]

Average
& 40.8\pmstd{0.7}
& 40.2\pmstd{1.4}
& \textbf{44.8}\pmstd{1.4}
& 44.3\pmstd{1.5} \\

\bottomrule
\end{tabular*}
\end{minipage}
\hfill
\begin{minipage}[t]{0.485\textwidth}
\centering
\captionof{table}{Average results over five representations on OGBench comparing offline OLS / BLS against online ReLA/LoLA.}
\label{tab:baselines_online}

\footnotesize
\setlength{\tabcolsep}{2.8pt}
\renewcommand{\arraystretch}{1.12}

\begin{tabular*}{\linewidth}{@{\extracolsep{\fill}}lcccc@{}}
\toprule
& OLS
& ReLA
& LoLA
& BLS \\
\midrule

antmaze-mn
& 65.9\pmstd{0.9}
& 49.7\pmstd{2.1}
& 64.9\pmstd{2.4}
& \textbf{79.2}\pmstd{2.8} \\

antmaze-ln
& 43.2\pmstd{1.4}
& 19.4\pmstd{2.5}
& 36.4\pmstd{2.6}
& \textbf{53.6}\pmstd{2.6} \\

antmaze-ms
& 38.9\pmstd{2.8}
& 33.3\pmstd{2.8}
& 39.5\pmstd{2.8}
& \textbf{48.3}\pmstd{2.1} \\

antmaze-ls
& 24.7\pmstd{1.6}
& 8.3\pmstd{1.7}
& 20.5\pmstd{3.0}
& \textbf{24.7}\pmstd{1.5} \\

antmaze-me
& 14.4\pmstd{1.4}
& 10.7\pmstd{1.5}
& 16.2\pmstd{2.3}
& \textbf{18.5}\pmstd{1.8} \\

cube-single
& 38.9\pmstd{1.3}
& 38.5\pmstd{1.2}
& 40.6\pmstd{1.5}
& \textbf{44.7}\pmstd{1.8} \\

cube-double
& 2.2\pmstd{0.7}
& 2.7\pmstd{0.5}
& 2.3\pmstd{0.6}
& \textbf{7.3}\pmstd{0.7} \\

scene
& 23.1\pmstd{1.1}
& 16.1\pmstd{2.8}
& 21.7\pmstd{2.3}
& \textbf{30.4}\pmstd{1.6} \\

puzzle-3$\times$3
& 6.9\pmstd{0.3}
& 7.9\pmstd{1.1}
& 7.0\pmstd{1.2}
& \textbf{9.0}\pmstd{0.7} \\

\midrule[\heavyrulewidth]

Average
& 28.7\pmstd{0.5}
& 20.8\pmstd{0.7}
& 27.7\pmstd{0.7}
& \textbf{35.1}\pmstd{0.6} \\

\bottomrule
\end{tabular*}
\end{minipage}

\end{table*}

We first compare \method{} against two offline zero-shot inference methods, OLS and ZOL, on OGBench. Since ZOL is specific to the FB representation~\citep{touati2021learning}, we use FB for all methods in this comparison. The trust-region loss of \method{} can use either OLS or ZOL as its anchor. We denote these variants as \method{}\at{OLS} and \method{}\at{ZOL}, respectively. As shown in Table~\ref{tab:baselines_offline}, both variants improve their corresponding anchors on average. \method{}\at{OLS} increases the average success rate from $40.8\%$ to $44.8\%$, while \method{}\at{ZOL} improves ZOL from $40.2\%$ to $44.3\%$, corresponding to gains of $4.0$ and $4.1$ percentage points.

We next compare our offline inference method against the online adaptation methods ReLA and LoLA. ReLA and LoLA use $20$ reward-labeled interaction episodes per task, whereas \method{} task inference requires no task-specific interaction. We evaluate all methods over 5 feature representations: Laplacian, ICVF, HILP, FB, and TD-JEPA, and report the average performance across them; per-representation results are provided in Appendix~\ref{app:ogbench_baselines}. As shown in Table~\ref{tab:baselines_online}, \method{} achieves an average success rate of $35.1\%$, outperforming ReLA ($20.8\%$) and LoLA ($27.7\%$) by $14.3$ and $7.4$ percentage points, respectively.

\subsection{How does \method{} compare with goal-conditioned RL?}
\begin{figure}[tb]
\centering
    \includegraphics[width=\textwidth]{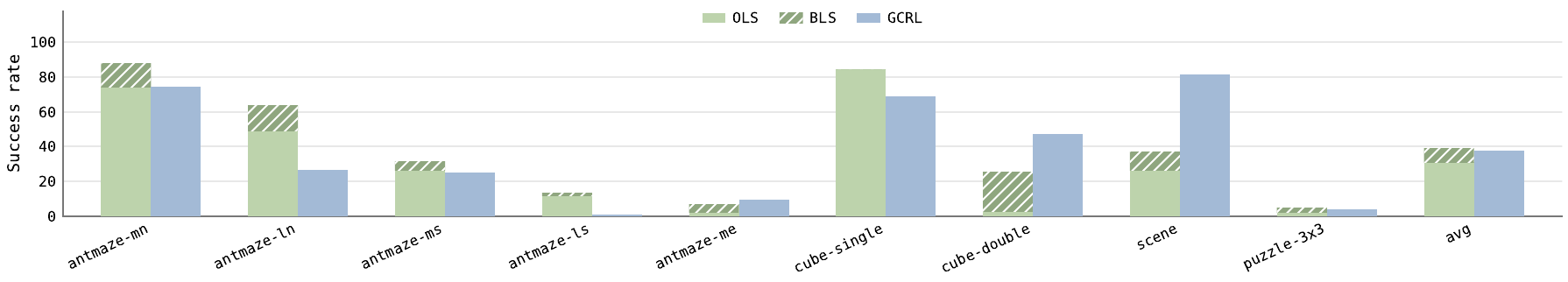}
   \caption{Comparison of OLS, \method{}, and GCRL on OGBench using HILP. The hatched segments indicate the improvement of \method{} over OLS. \method{} achieves performance comparable to GCRL on average, while their relative performance varies across tasks.}
   \label{fig:gcrl}

\end{figure}
We report the success rate of \method{} against goal-conditioned RL (GCRL) on the proprioceptive setting of OGBench. For a fair comparison, we use HILP for both zero-shot RL and GCRL, as HILP supports both formulations within the same representation framework. GCRL provides a strong reference point for goal-reaching tasks because it is trained directly on the goal-conditioned objective and receives the goal explicitly, whereas zero-shot RL must infer a task vector from the reward function. As shown in Figure~\ref{fig:gcrl}, \method{} achieves an average success rate of $39.1\%$, comparable to GCRL at $37.7\%$, while OLS reaches only $31.1\%$. The relative performance varies across tasks: \method{} performs particularly well on several navigation tasks, whereas GCRL is stronger on some manipulation tasks such as \textit{cube-double} and \textit{scene}. These results show that improved task inference can substantially close the performance gap between zero-shot RL and direct goal-conditioned control.

\subsection{How does \method{} perform on dense-reward tasks?}
\label{sec:dense}
\begin{figure}[t]
    \centering
    \includegraphics[width=\textwidth]{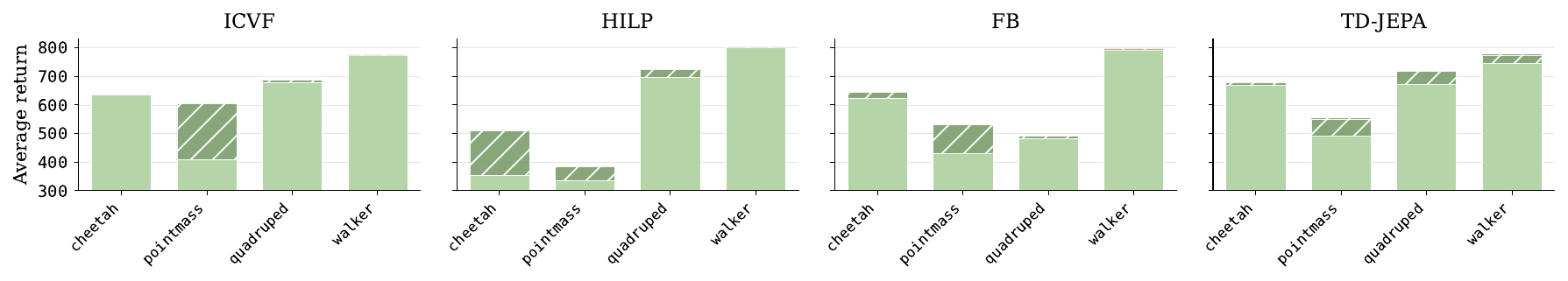}\\[4pt]
    \includegraphics[width=\textwidth]{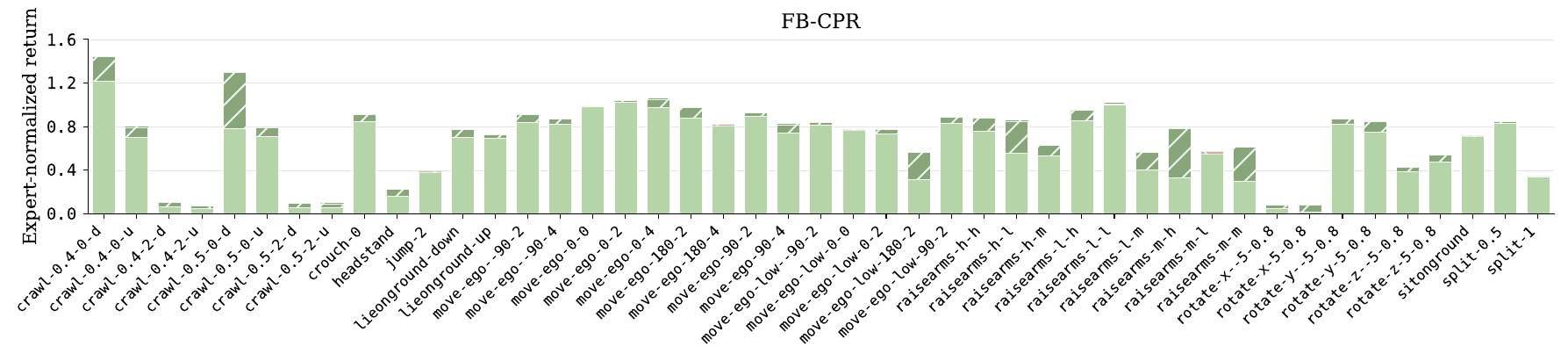}
    \caption{Dense-reward results on DMC (top) and HumEnv (bottom). Solid bars show the return of the OLS policy, and hatched segments indicate the change after applying BLS. DMC reports average return across tasks for each domain and feature representation, while HumEnv reports expert-normalized return on 45 motion-tracking tasks, with FB-CPR feature representations.}
    \label{fig:dense_bar}
\end{figure}

We next investigate whether \method{} remains effective on dense-reward tasks. We evaluate on DMC with 20 tasks across 4 domains and on HumEnv with 45 humanoid control tasks. Since dense rewards do not naturally define a sparse goal set, we construct $G$ from the high-reward region of each benchmark. We define $G$ using the top $5\%$ high-reward region and let the remaining states form $N$. Full configuration and evaluation details are provided in Appendix~\ref{app:implementation}.

The results are summarized in Figure~\ref{fig:dense_bar}. Although the gains are generally smaller than on sparse-reward tasks, \method{} remains effective on both benchmarks. On DMC, after averaging over the four feature representations (ICVF, HILP, FB, and TD-JEPA), \method{} improves 18 of 20 tasks. On HumEnv, using the FB-CPR representation~\citep{tirinzoni2025zero}, it improves 42 of 45 tasks, with the largest gains on \textit{crawl} and \textit{raisearms}. Overall, these results show that \method{} also generalizes effectively to dense-reward tasks across different control domains.

\subsection{Ablations and Further Analysis}
\label{subsec:ablation}
We further analyze \method{} in terms of (1) sensitivity to $\lambda$, (2) robustness to the construction of the high- and low-reward sets $G$ and $N$, (3) computational overhead, and (4) robustness to the feature dimension $d$.

\begin{figure}[tb]
\centering
   \includegraphics[width=\textwidth]{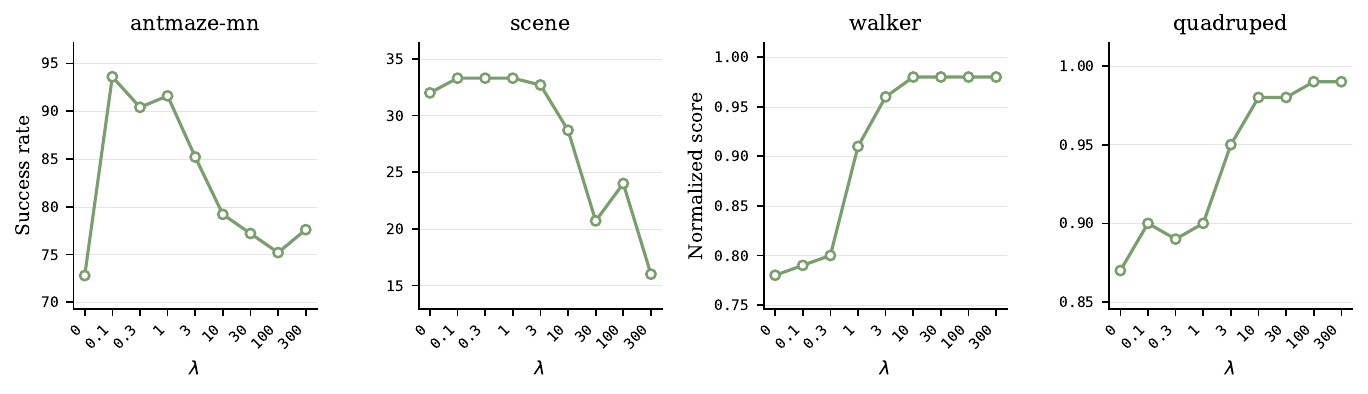}
   \caption{Sensitivity to $\lambda$ with the HILP encoder: sparse rewards (left two) prefer a small $\lambda$, and dense rewards (right two) prefer a large one.}
   \label{fig:lambda}
\end{figure}
\noindent\textbf{The sensitivity to the choice of $\lambda$.}
We observed that small values of $\lambda$ are suited for sparse-reward tasks, and vice versa.  We conducted experiments on two sparse-reward tasks (\textit{antmaze} and \textit{scene}) from OGBench, and all tasks with the \textit{walker} and \textit{quadruped} embodiments from DMC. We report the performance obtained with each $\lambda$ selected from the range $\{0, 0.1, 0.3, 1, 3, 10, 30, 100, 300\}$.  Results are summarized in Figure~\ref{fig:lambda}.  We show that keeping the task vector close to OLS's solution (larger $\lambda$) is fine for dense-reward tasks, but it soon becomes ineffective when the rewards get sparser.  Enlarging the margins between low- and high-value rewards significantly improves the task performance under the sparse-reward regime.

\begin{figure}[t]
\centering
  \includegraphics[width=0.52\textwidth]{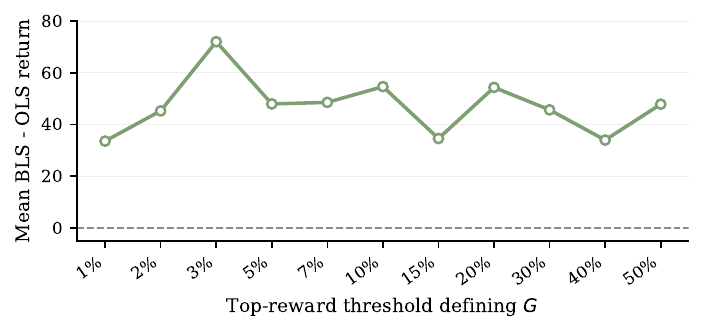}
  \caption{
    Sensitivity to the reward threshold for defining $G$ on DMC with TD-JEPA. We report the paired mean return gain of \method{} over OLS across $20$ tasks as the threshold varies.
  }
  \label{fig:g_threshold_sensitivity}

\end{figure}
\noindent\textbf{Sensitivity to the reward threshold.}
For dense-reward tasks, we define $G$ using a threshold on the reward distribution. We evaluate the sensitivity to this threshold on all $20$ DMC tasks with TD-JEPA, varying the cutoff from the top $1\%$ to the top $50\%$ of reward samples. As shown in Figure~\ref{fig:g_threshold_sensitivity}, \method{} achieves a positive paired mean improvement over OLS across the entire range. For thresholds between the top $3\%$ and top $50\%$, the mean improvement ranges from $34$ to $72$ return points, showing substantial gains across threshold choices. Thus, \method{} does not require careful tuning of the reward threshold used to construct $G$.

\noindent\textbf{The computational efficiency.}
While \method{} adopts gradient-based optimization, it incurs little overhead during inference.  Using an Nvidia RTX 4090 GPU card, optimizing the task vector for $500$ steps costs less than $1$ second on OGBench and $1.4$ seconds on HumEnv.  Notably, the inference procedure is performed once before executing the task.  Therefore, the incurred latency is negligible.
\begin{table}[H]
\centering
\caption{Success rate as the feature dimension $d$ varies.}
\label{tab:dim}
\small
\setlength{\tabcolsep}{6pt}
\renewcommand{\arraystretch}{1.25}
\begin{tabular}{@{} l cccccc @{}}
\toprule
$d$ & 25 & 50 & 100 & 200 & 500 & 1000 \\
\midrule
OLS & 70.9\pmstd{2.1} & 74.3\pmstd{1.4} & 59.2\pmstd{1.7} & 25.3\pmstd{1.5} & 9.2\pmstd{0.4} & 22.0\pmstd{5.7} \\
BLS & \textbf{84.7}\pmstd{2.0} & \textbf{87.7}\pmstd{1.5} & \textbf{87.5}\pmstd{2.0} & \textbf{82.1}\pmstd{9.0} & \textbf{79.6}\pmstd{8.4} & \textbf{82.7}\pmstd{16.7} \\
\bottomrule
\end{tabular}
\end{table}
\noindent\textbf{Robustness to feature dimension.}
A larger feature dimension $d$ increases representation capacity, but does not guarantee a better learned representation. This distinction is important for OLS, whose effectiveness depends on how well the downstream reward can be represented by the learned feature basis. We test this by training HILP on \textit{antmaze-mn} with $d \in \{25, 50, 100, 200, 500, 1000\}$ while keeping the number of reward samples fixed for task inference. As shown in Table~\ref{tab:dim}, increasing $d$ does not make OLS more reliable: its success rate peaks at $74.3\%$ with $d=50$ and then drops sharply to $9.2\%$ at $d=500$. In contrast, \method{} outperforms OLS at every $d$, remaining between $79\%$ and $88\%$ across all dimensions. These results show that increasing feature capacity alone cannot resolve the limitation of OLS, whereas \method{} remains effective even when the learned representation does not yield a reliable OLS solution.

\section{Conclusion}
In this work, we study the limitations of OLS-based task inference in Behavioral Foundation Models, showing that the quality of an inferred task vector depends not only on reward reconstruction error but also on the successor measure induced by the reconstructed reward. We propose \method{}, which refines the task vector with a soft-margin objective while staying close to a given inference anchor, and consistently outperforms existing task inference methods with little additional computation. While these results are encouraging, the soft-margin objective remains an empirical surrogate for successor-measure mismatch, and its weight against the anchor depends on reward sparsity. We hope our results encourage future work to consider the induced policy behavior, rather than reward reconstruction alone, when designing task inference methods for BFMs.

\section*{Acknowledgment}
This work was supported by the AMD University Program AI \& HPC Cluster. We further acknowledge Kuo-Guang Tsai for his technical support on the AMD system cluster infrastructure. This research was also supported by the National Science and Technology Council (NSTC), Taiwan, under Grants 115-2222-E-002-006-, 115-2634-F-002-001-,  115-2218-E-002-026-, 115-2634-F-002-012-, 115-2628-E-A49-001, and 115-2634-F-A49-011.

\section*{AI Use Statement}
We used large language models to assist with checking the proof, implementing and debugging experimental code, interpreting results, and polishing the text. All ideas, methods, and claims originate from the authors, and every AI-assisted proof, code path, and passage was verified by us.

\section*{Reproducibility Statement}
We provide the details necessary to reproduce our results throughout the paper and appendix. The \method{} objective and optimization procedure are specified in Section~\ref{subsec:mtd_algorithm} and Algorithm~\ref{alg:bls}, with complete proofs provided in Appendix~\ref{app:proof}. Appendix~\ref{app:implementation} reports the implementation and evaluation details, including task-inference data, optimization hyperparameters, benchmark-specific configurations, evaluation protocols, and random-seed settings. Additional task-level results and robustness analyses are provided in Appendix~\ref{app:additional_results}.


{
    \small
    \setlength{\bibsep}{2pt}
    \bibliographystyle{ieeenat_fullname}
    \bibliography{bibs/main}
}

\restoreappendixlayout
\appendix
\newpage
\section{Proofs}
\label{app:proof}

This section provides the proof of Theorem~\ref{thm:main}.
Throughout this section, $\mathcal{S}$ and $\mathcal{A}$ are finite and
$\gamma \in (0,1)$. Recall that $\rho$ denotes the initial state distribution.
For a policy $\pi$, we define its policy-averaged successor measure as
$\bar{M}^{\pi}(s^+ \mid s)
= \mathbb{E}_{a \sim \pi(\cdot \mid s)}
[M^{\pi}(s^+ \mid s,a)]$.
For a kernel $K(\cdot \mid s)$ and a function $f$, we write
$\langle K, f \rangle_{\rho}
= \mathbb{E}_{s \sim \rho}
\sum_{s^+} K(s^+ \mid s) f(s^+)$,
and define the mixed norm
$\lVert K \rVert_{1}
= \sup_{s} \sum_{s^+} |K(s^+ \mid s)|$.
Under this notation,
$V_r^\pi = \langle \bar{M}^{\pi}, r \rangle_\rho$.

We restate Theorem~\ref{thm:main} for convenience.

\begin{theorem}[Suboptimality gap]
Let $\pi_r$ and $\pi_{\hat r}$ be the optimal policies for rewards $r$ and
$\hat r$, respectively, and let $\bar{M}^{\pi_r}$ and
$\bar{M}^{\pi_{\hat r}}$ be their corresponding policy-averaged successor
measures. Define
$\Delta M(s^+ \mid s)
\coloneq \bar{M}^{\pi_{\hat r}}(s^+ \mid s)
- \bar{M}^{\pi_r}(s^+ \mid s)$
as the residual between the two successor measures, and
$\Delta r(s) \coloneq\hat r(s) - r(s)$
as the residual between the two rewards. Then,
\[
\mathrm{SubOpt}(r,\hat r)
\le
\lVert \Delta M \rVert_1
\cdot
\lVert \Delta r \rVert_\infty.
\]
\end{theorem}

\begin{proof}
 We first relate the suboptimality gap to $\Delta M$ and $\Delta r$.  Since $\pi_{\hat r}$ is optimal for $\hat r$,
\begin{align}\langle \Delta M,\,  \hat{r}  \rangle_{\rho}=V_{\hat r}^{\pi_{\hat r}}-
V_{\hat r}^{\pi_r}\ge 0\label{eq:optimal}\end{align}. By Eqn. \ref{eq:optimal} and definition of $\Delta M$ and $\Delta r$, we have
 \begin{align}
    \mathrm{SubOpt}(r, \hat{r})
    &= -\langle \Delta M, \, r\rangle_{\rho}\\
    &=\langle \Delta M, \, \Delta r\rangle_{\rho}-\langle \Delta M,\,  \hat{r}  \rangle_{\rho}\\
    &\le \langle \Delta M, \, \Delta r\rangle_{\rho} \label{eq:drop}
 \end{align}
Applying H\"older's inequality to each $s$ gives
\begin{align}
\left|
\langle \Delta M, \Delta r
\rangle_\rho
\right|
&\le
\mathbb{E}_{s\sim\rho}
\sum_{s^+\in\mathcal S}
|\Delta M(s^+\mid s)|
|\Delta r(s^+)| \\
&\le
\lVert \Delta M\rVert_1
\lVert \Delta r\rVert_\infty .
\label{eq:holder}
\end{align}

Combining Eqn.~\ref{eq:drop} and Eqn.~\ref{eq:holder}, we conclude that
\[
\mathrm{SubOpt}(r,\hat r)
\le
\lVert \Delta M\rVert_1
\lVert \Delta r\rVert_\infty
.
\]
\end{proof}
\begin{table*}[t]
\caption{OGBench success rates (\%) across four feature representations.}
\label{tab:app_ogbench_repr}
\centering
\scriptsize
\renewcommand{\arraystretch}{1.02}

\begin{minipage}[t]{0.485\textwidth}
\centering
\setlength{\tabcolsep}{2pt}

\begin{tabular*}{\linewidth}{
@{\extracolsep{\fill}} l c c c c @{}
}
\toprule
\multicolumn{5}{c}{Laplacian} \\
\midrule
Task & OLS & BLS & ReLA & LoLA \\
\midrule
antmaze-mn
& 48.1\pmstd{2.9}
& \textbf{67.3}\pmstd{8.1}
& 39.6\pmstd{8.0}
& 48.5\pmstd{9.2} \\
antmaze-ln
& 7.6\pmstd{1.8}
& \textbf{10.9}\pmstd{2.9}
& 7.3\pmstd{6.1}
& 5.6\pmstd{1.1} \\
antmaze-ms
& 15.1\pmstd{7.8}
& \textbf{34.7}\pmstd{3.4}
& 12.9\pmstd{8.5}
& 10.8\pmstd{4.2} \\
antmaze-ls
& 15.2\pmstd{3.9}
& \textbf{15.3}\pmstd{0.6}
& 4.9\pmstd{5.8}
& 12.9\pmstd{5.7} \\
antmaze-me
& 3.7\pmstd{3.3}
& \textbf{7.6}\pmstd{1.7}
& 0.0\pmstd{0.0}
& 3.5\pmstd{3.6} \\
cube-single
& 19.9\pmstd{1.3}
& 18.3\pmstd{4.6}
& 18.5\pmstd{3.4}
& \textbf{23.6}\pmstd{1.4} \\
cube-double
& 1.3\pmstd{0.6}
& 2.0\pmstd{1.8}
& \textbf{2.4}\pmstd{0.7}
& 1.9\pmstd{0.8} \\
scene
& 4.3\pmstd{1.4}
& 4.0\pmstd{1.7}
& 4.3\pmstd{1.2}
& \textbf{5.3}\pmstd{0.8} \\
puzzle-3$\times$3
& 1.5\pmstd{0.6}
& 1.6\pmstd{0.8}
& 1.7\pmstd{0.6}
& \textbf{3.5}\pmstd{0.6} \\
\midrule
Average
& 13.0
& \textbf{18.0}
& 10.2
& 12.8 \\
\bottomrule
\end{tabular*}
\end{minipage}
\hfill
\begin{minipage}[t]{0.485\textwidth}
\centering
\setlength{\tabcolsep}{2pt}

\begin{tabular*}{\linewidth}{
@{\extracolsep{\fill}} l c c c c @{}
}
\toprule
\multicolumn{5}{c}{ICVF} \\
\midrule
Task & OLS & BLS & ReLA & LoLA \\
\midrule
antmaze-mn
& 66.3\pmstd{1.2}
& \textbf{69.2}\pmstd{6.7}
& 51.5\pmstd{3.2}
& 53.6\pmstd{2.4} \\
antmaze-ln
& 56.0\pmstd{4.1}
& \textbf{65.5}\pmstd{4.9}
& 25.6\pmstd{10.0}
& 39.5\pmstd{5.4} \\
antmaze-ms
& 28.8\pmstd{5.6}
& \textbf{36.5}\pmstd{9.1}
& 18.3\pmstd{4.0}
& 33.5\pmstd{5.4} \\
antmaze-ls
& 5.5\pmstd{0.8}
& 4.0\pmstd{0.7}
& \textbf{6.8}\pmstd{6.2}
& \textbf{6.8}\pmstd{2.9} \\
antmaze-me
& \textbf{0.3}\pmstd{0.5}
& 0.0\pmstd{0.0}
& 0.0\pmstd{0.0}
& 0.0\pmstd{0.0} \\
cube-single
& \textbf{19.5}\pmstd{1.2}
& 13.3\pmstd{0.5}
& 13.7\pmstd{0.6}
& 19.3\pmstd{2.0} \\
cube-double
& 3.1\pmstd{1.5}
& 3.5\pmstd{0.8}
& 2.8\pmstd{1.1}
& \textbf{3.6}\pmstd{1.6} \\
scene
& 39.3\pmstd{3.2}
& \textbf{47.7}\pmstd{3.7}
& 37.2\pmstd{2.2}
& 37.5\pmstd{5.9} \\
puzzle-3$\times$3
& 13.6\pmstd{0.7}
& \textbf{15.3}\pmstd{1.4}
& 8.7\pmstd{1.6}
& 12.4\pmstd{3.9} \\
\midrule
Average
& 25.8
& \textbf{28.3}
& 18.3
& 22.9 \\
\bottomrule
\end{tabular*}
\end{minipage}

\vspace{8pt}

\begin{minipage}[t]{0.485\textwidth}
\centering
\setlength{\tabcolsep}{2pt}

\begin{tabular*}{\linewidth}{
@{\extracolsep{\fill}} l c c c c @{}
}
\toprule
\multicolumn{5}{c}{HILP} \\
\midrule
Task & OLS & BLS & ReLA & LoLA \\
\midrule
antmaze-mn
& 74.3\pmstd{1.4}
& \textbf{87.7}\pmstd{1.5}
& 65.9\pmstd{2.0}
& 76.7\pmstd{3.0} \\
antmaze-ln
& 49.2\pmstd{2.1}
& \textbf{63.7}\pmstd{1.8}
& 24.5\pmstd{2.9}
& 42.5\pmstd{9.5} \\
antmaze-ms
& 26.3\pmstd{9.1}
& \textbf{31.3}\pmstd{2.3}
& 12.7\pmstd{9.8}
& 23.7\pmstd{11.5} \\
antmaze-ls
& 11.9\pmstd{1.2}
& \textbf{13.5}\pmstd{0.8}
& 0.7\pmstd{0.6}
& 4.9\pmstd{6.9} \\
antmaze-me
& 2.4\pmstd{2.6}
& 6.7\pmstd{1.5}
& \textbf{8.0}\pmstd{7.5}
& 7.3\pmstd{5.3} \\
cube-single
& \textbf{84.5}\pmstd{5.7}
& 82.0\pmstd{6.7}
& \textbf{84.5}\pmstd{3.5}
& 83.9\pmstd{6.0} \\
cube-double
& 2.8\pmstd{3.1}
& \textbf{25.6}\pmstd{2.1}
& 1.3\pmstd{2.0}
& 2.9\pmstd{2.0} \\
scene
& 26.3\pmstd{2.7}
& \textbf{36.8}\pmstd{2.2}
& 12.5\pmstd{10.9}
& 16.9\pmstd{9.4} \\
puzzle-3$\times$3
& 2.4\pmstd{0.0}
& \textbf{4.5}\pmstd{1.5}
& 3.1\pmstd{2.7}
& 3.2\pmstd{0.4} \\
\midrule
Average
& 31.1
& \textbf{39.1}
& 23.7
& 29.1 \\
\bottomrule
\end{tabular*}
\end{minipage}
\hfill
\begin{minipage}[t]{0.485\textwidth}
\centering
\setlength{\tabcolsep}{2pt}

\begin{tabular*}{\linewidth}{
@{\extracolsep{\fill}} l c c c c @{}
}
\toprule
\multicolumn{5}{c}{TD-JEPA} \\
\midrule
Task & OLS & BLS & ReLA & LoLA \\
\midrule
antmaze-mn
& 63.1\pmstd{0.8}
& \textbf{83.2}\pmstd{7.7}
& 20.5\pmstd{5.1}
& 65.2\pmstd{6.6} \\
antmaze-ln
& 50.7\pmstd{2.4}
& \textbf{64.3}\pmstd{7.0}
& 0.0\pmstd{0.0}
& 41.7\pmstd{3.4} \\
antmaze-ms
& 54.5\pmstd{4.8}
& 59.5\pmstd{2.4}
& 46.5\pmstd{1.2}
& \textbf{59.7}\pmstd{3.4} \\
antmaze-ls
& 39.2\pmstd{5.9}
& \textbf{45.1}\pmstd{5.9}
& 0.0\pmstd{0.0}
& 24.0\pmstd{11.0} \\
antmaze-me
& 15.1\pmstd{5.5}
& \textbf{29.2}\pmstd{7.0}
& 0.0\pmstd{0.0}
& 22.8\pmstd{6.5} \\
cube-single
& 20.3\pmstd{2.9}
& \textbf{55.9}\pmstd{3.0}
& 21.5\pmstd{0.8}
& 21.6\pmstd{1.8} \\
cube-double
& 1.7\pmstd{0.5}
& 2.1\pmstd{1.2}
& \textbf{4.7}\pmstd{1.0}
& 0.8\pmstd{0.8} \\
scene
& 37.9\pmstd{3.2}
& \textbf{54.0}\pmstd{6.8}
& 11.5\pmstd{8.7}
& 40.9\pmstd{2.0} \\
puzzle-3$\times$3
& 11.2\pmstd{0.4}
& \textbf{14.1}\pmstd{2.1}
& 10.9\pmstd{3.7}
& 10.8\pmstd{3.0} \\
\midrule
Average
& 32.6
& \textbf{45.3}
& 12.8
& 32.0 \\
\bottomrule
\end{tabular*}
\end{minipage}

\end{table*}
\section{Implementation and Evaluation Details}
\label{app:implementation}

\paragraph{Common inference and optimization setup.}
For OGBench and DMC, we use $10{,}000$ transitions sampled from the offline replay buffer for task inference. We follow the sampling procedure of each pretrained representation. The corresponding next states are used to compute the feature representations and rewards, and the same inference batch is used for both OLS and BLS. This setting is used for both proprioceptive and pixel-based OGBench experiments. For HumEnv, we use $50{,}000$ inference samples. We optimize the task vector for $T=500$ steps using Adam with a learning rate of $\alpha=0.03$. After each gradient update, the task vector is projected onto the unit sphere. Reconstructed rewards are normalized by the standard deviation of the rewards reconstructed from the anchor task vector, and the normalized scores are used in the soft-margin objective.

\paragraph{Benchmark-specific configurations.}
On OGBench, $G$ consists of states attaining the maximum reward, and the remaining states form $N$. For each feature representation and observation modality, we use a fixed $\lambda$ across all tasks. For proprioceptive Laplacian, ICVF, HILP, FB, and TD-JEPA, we use $\lambda = 10, 2, 3, 0.3,$ and $2$, respectively. For their pixel-based counterparts, we use $\lambda = 2, 10, 2, 20,$ and $3$, respectively. For dense-reward tasks on DMC and HumEnv motion tracking, we define $G$ as the top $5\%$ of reward samples and let the remaining samples form $N$. Since the same $\lambda$ can induce different deviations from the anchor across tasks, we parameterize the trust-region strength by the cosine similarity between the refined and anchor task vectors. We use a fixed cosine similarity target of $0.95$ for all such experiments, where $\lambda$ is chosen to match this target between the refined and anchor task vectors. For HumEnv goal-reaching, an explicit reward is not provided. We therefore define $G$ as the $100$ states in the offline buffer closest to the target pose and use the benchmark proximity function as the surrogate reward. We use a fixed $\lambda=20$ for all target poses.

\paragraph{Evaluation details.}
Unless otherwise stated, all reported results are averaged over three random seeds. On OGBench, we report success rate; on DMC, episode return; and on HumEnv, expert-normalized return. The five official FB-CPR checkpoints used in the checkpoint-robustness experiment constitute a separate source of variation: they share the same architecture and training recipe and differ only in their pretraining seeds. For this experiment, the reported mean and standard deviation are computed across the five checkpoint-level aggregate scores.

\paragraph{Goal-conditioned RL comparison.}
For the comparison with goal-conditioned RL (GCRL) in Section~\ref{sec:experiments}, we evaluate all methods on the same nine proprioceptive OGBench domains, each containing five evaluation tasks. The zero-shot methods and GCRL use HILP as their representation framework but are trained separately according to their respective objectives. GCRL is trained directly with goal-conditioned supervision and receives the target goal state at evaluation time, whereas OLS and BLS infer a task vector from reward-labeled samples. The main GCRL baseline follows its standard goal-conditioned training configuration and is trained from scratch for each domain and random seed using the full offline training dataset. Each model is trained for $1$M gradient steps, and we report only the final checkpoint rather than selecting the best intermediate checkpoint. We use three independent training seeds and evaluate each checkpoint deterministically over $50$ episodes for each of the five tasks, using the OGBench success metric. For OLS and BLS, we use the same fixed HILP inference configuration as in the main OGBench experiments.

\section{Additional Experimental Results}
\label{app:additional_results}

\subsection{Detailed Comparison with Baselines on OGBench}
\label{app:ogbench_baselines}

We report the per-representation results underlying the baseline comparison in \cref{tab:app_ogbench_repr,tab:app_fb}. ReLA and LoLA use $20$ reward-labeled interaction episodes per task, whereas OLS and BLS require no task-specific interaction. ZOL is reported only with FB, for which it is designed.

\begin{table*}[t]
\caption{
OGBench success rates (\%) with the FB representation, including the
FB-specific ZOL baseline and the online adaptation baselines ReLA and LoLA.
}
\label{tab:app_fb}
\centering
\scriptsize
\setlength{\tabcolsep}{5pt}
\renewcommand{\arraystretch}{1.08}

\begin{tabular*}{0.92\textwidth}{
@{\extracolsep{\fill}}
l c c c c c c
@{}
}
\toprule
Task
& OLS
& ZOL
& BLS\at{OLS}
& BLS\at{ZOL}
& ReLA
& LoLA
\\
\midrule

antmaze-mn
& 77.5\pmstd{2.6}
& 84.5\pmstd{1.6}
& \textbf{88.8}\pmstd{4.2}
& 84.4\pmstd{4.5}
& 70.9\pmstd{3.0}
& 80.5\pmstd{2.4}
\\

antmaze-ln
& 52.4\pmstd{4.0}
& 41.5\pmstd{0.8}
& 63.5\pmstd{9.2}
& \textbf{68.7}\pmstd{2.4}
& 39.7\pmstd{3.1}
& 52.8\pmstd{5.8}
\\

antmaze-ms
& 69.9\pmstd{1.7}
& 51.9\pmstd{1.2}
& \textbf{79.6}\pmstd{2.1}
& 68.1\pmstd{3.6}
& 76.1\pmstd{1.9}
& 70.0\pmstd{2.2}
\\

antmaze-ls
& 51.5\pmstd{2.9}
& 37.9\pmstd{5.1}
& 45.7\pmstd{4.6}
& 36.0\pmstd{4.6}
& 29.2\pmstd{1.1}
& \textbf{53.7}\pmstd{3.8}
\\

antmaze-me
& \textbf{50.3}\pmstd{2.0}
& 49.7\pmstd{3.1}
& 49.2\pmstd{4.9}
& 48.8\pmstd{1.4}
& 45.5\pmstd{0.5}
& 47.6\pmstd{7.3}
\\

cube-single
& 50.4\pmstd{0.8}
& 46.4\pmstd{6.8}
& 54.0\pmstd{1.8}
& 43.7\pmstd{6.5}
& \textbf{54.5}\pmstd{2.9}
& \textbf{54.5}\pmstd{3.6}
\\

cube-double
& 1.9\pmstd{1.0}
& 2.8\pmstd{0.0}
& \textbf{3.1}\pmstd{1.2}
& 1.7\pmstd{1.0}
& 2.4\pmstd{0.7}
& 2.3\pmstd{0.8}
\\

scene
& 7.6\pmstd{0.7}
& 33.5\pmstd{7.9}
& 9.6\pmstd{0.4}
& \textbf{35.3}\pmstd{8.4}
& 15.2\pmstd{1.4}
& 8.1\pmstd{1.2}
\\

puzzle-3$\times$3
& 5.7\pmstd{1.2}
& 13.7\pmstd{1.2}
& 9.6\pmstd{1.8}
& 11.7\pmstd{1.6}
& \textbf{15.3}\pmstd{2.3}
& 4.9\pmstd{3.0}
\\

\midrule
Average
& 40.8
& 40.2
& \textbf{44.8}
& 44.3
& 38.8
& 41.6
\\
\bottomrule
\end{tabular*}
\end{table*}

\subsection{Task-Level Results on DMC}
\label{app:dmc_task_results}

Table~\ref{tab:app_dmc_tasks} reports the task-level results underlying
Figure~\ref{fig:dense_bar}. Each cell reports OLS / BLS episode return.
The final column averages over ICVF, HILP, FB, and TD-JEPA for each task.
Under this task-wise aggregation, BLS improves $18$ of the $20$ DMC tasks.

\begin{table*}[t]
\caption{Task-level DMC episode return under the trust-region configuration used in the main experiments. Each cell reports OLS / BLS, averaged over three random seeds. The final column averages over the four feature representations used in Figure~\ref{fig:dense_bar}.}
\label{tab:app_dmc_tasks}
\centering
\footnotesize
\setlength{\tabcolsep}{2.5pt}
\renewcommand{\arraystretch}{1.08}

\begin{tabular*}{\textwidth}{
@{\extracolsep{\fill}} l c c c c c @{}
}
\toprule
Task & ICVF & HILP & FB & TD-JEPA & Avg. \\
\midrule
walker-stand
& 988 / 961
& 954 / 959
& 938 / 944
& 956 / 966
& 959.0 / 957.5 \\

walker-walk
& 850 / 884
& 943 / 932
& 921 / 919
& 754 / 834
& 867.0 / \textbf{892.2} \\

walker-run
& 307 / 279
& 332 / 346
& 353 / 333
& 316 / 337
& 327.0 / 323.8 \\

walker-spin
& 960 / 972
& 982 / 968
& 978 / 977
& 954 / 962
& 968.5 / \textbf{969.8} \\

\midrule
cheetah-walk
& 979 / 979
& 255 / 651
& 876 / 974
& 967 / 983
& 769.2 / \textbf{896.8} \\

cheetah-walk-backward
& 986 / 968
& 874 / 969
& 985 / 963
& 981 / 982
& 956.5 / \textbf{970.5} \\

cheetah-run
& 275 / 284
& 86 / 161
& 317 / 339
& 367 / 382
& 261.2 / \textbf{291.5} \\

cheetah-run-backward
& 301 / 311
& 196 / 257
& 317 / 304
& 362 / 364
& 294.0 / \textbf{309.0} \\

\midrule
quadruped-stand
& 971 / 974
& 962 / 961
& 597 / 638
& 971 / 962
& 875.2 / \textbf{883.8} \\

quadruped-walk
& 501 / 524
& 561 / 653
& 482 / 467
& 535 / 571
& 519.8 / \textbf{553.8} \\

quadruped-run
& 489 / 491
& 489 / 505
& 335 / 340
& 442 / 508
& 438.8 / \textbf{461.0} \\

quadruped-jump
& 760 / 763
& 777 / 776
& 520 / 523
& 737 / 825
& 698.5 / \textbf{721.8} \\

\midrule
pointmass-reach-top-left
& 919 / 939
& 961 / 965
& 823 / 819
& 615 / 671
& 829.5 / \textbf{848.5} \\

pointmass-reach-top-right
& 87 / 710
& 145 / 141
& 828 / 801
& 498 / 545
& 389.5 / \textbf{549.2} \\

pointmass-reach-bottom-left
& 486 / 738
& 687 / 680
& 827 / 826
& 573 / 670
& 643.2 / \textbf{728.5} \\

pointmass-reach-bottom-right
& 402 / 565
& 2 / 0
& 131 / 105
& 192 / 74
& 181.8 / \textbf{186.0} \\

pointmass-reach-bottom-left-long
& 212 / 274
& 4 / 201
& -43 / 85
& 161 / 377
& 83.5 / \textbf{234.2} \\

pointmass-loop
& 788 / 809
& 346 / 549
& 354 / 615
& 703 / 784
& 547.8 / \textbf{689.2} \\

pointmass-square
& 181 / 612
& 488 / 481
& 397 / 862
& 778 / 826
& 461.0 / \textbf{695.2} \\

pointmass-fast-slow
& 195 / 198
& 52 / 58
& 121 / 133
& 413 / 460
& 195.2 / \textbf{212.2} \\

\midrule
Average
& 582 / \textbf{662}
& 505 / \textbf{561}
& 553 / \textbf{599}
& 614 / \textbf{654}
& 563.3 / \textbf{618.7} \\
\bottomrule
\end{tabular*}
\end{table*}

\subsection{Refinement from a Strong Task-Vector Anchor}
\label{app:strong_anchor}

We study whether \method{} can further improve a task vector that already retrieves a strong zero-shot policy and is obtained independently of OLS. We conduct experiments on $50$ goal-reaching tasks in HumEnv using FB-CPR representations~\citep{tirinzoni2025zero}. We compare against four baselines: OLS, $\Psi_G$, WR, and $B(g)$, where $B(g)$ directly uses the backward embedding of the goal state as the task vector. Following~\citet{tirinzoni2025zero}, we define weighted reward inference (WR) as $z_{\mathrm{WR}} \propto \mathbb{E}_{s\sim\mathcal D}\left[e^{10r(s)} B(s)r(s)\right]$. We evaluate performance using three metrics: success rate, \textit{proximity}, which measures the average fraction of time the agent maintains the goal pose, and \textit{distance}, which measures the distance to the target pose.
\begin{table*}[t]
\caption{Goal-reaching on HumEnv ($50$ target poses, FB-CPR).}
\label{tab:app_strong_anchor}
\centering
\small
\setlength{\tabcolsep}{6pt}
\renewcommand{\arraystretch}{1.25}
\begin{tabular}{@{} l ccccc @{}}
\toprule
& OLS & $\Psi_G$ & WR & $B(g)$ & BLS\at{$B(g)$} \\
\midrule
success $\uparrow$
& 0.144\pmstd{0.020}
& 0.164\pmstd{0.008}
& 0.148\pmstd{0.032}
& 0.500\pmstd{0.022}
& \textbf{0.516}\pmstd{0.032} \\

proximity $\uparrow$
& 0.455\pmstd{0.023}
& 0.460\pmstd{0.022}
& 0.453\pmstd{0.022}
& 0.681\pmstd{0.009}
& \textbf{0.699}\pmstd{0.009} \\

distance $\downarrow$
& 3.559\pmstd{0.116}
& 3.309\pmstd{0.116}
& 3.533\pmstd{0.094}
& 2.668\pmstd{0.029}
& \textbf{2.630}\pmstd{0.038} \\
\bottomrule
\end{tabular}
\end{table*}

As shown in Table~\ref{tab:app_strong_anchor}, $B(g)$ is the strongest baseline, achieving a success rate of $0.500$, substantially outperforming the second-best baseline, $\Psi_G$, at $0.164$. We therefore use $B(g)$ as the anchor for \method{}, denoted as \method{}\at{$B(g)$}, with a fixed $\lambda=20$. Despite starting from this strong anchor, \method{} consistently improves all three metrics: success increases from $0.500$ to $0.516$, proximity from $0.681$ to $0.699$, and distance decreases from $2.668$ to $2.630$. These results show that \method{} is not limited to correcting OLS and can further refine a strong task vector obtained from a different inference mechanism.

\subsection{Results across FB-CPR Checkpoints}
\label{app:fb_cpr_checkpoints}

We evaluate \method{} on the five official FB-CPR checkpoints to examine whether its improvements are robust to variations in representation pretraining. These checkpoints share the same architecture and training recipe but are trained with different random seeds. For each checkpoint, we evaluate all 45 HumEnv tasks and compare against OLS, weighted regression (WR), and $\Psi_G$, where $\Psi_G$ computes the task vector by averaging the backward embeddings over the goal set $G$. We apply \method{} using either OLS or WR as the anchor, denoted as \method{}\at{OLS} and \method{}\at{WR}, respectively. Table~\ref{tab:app_checkpoints} reports the individual results for all five checkpoints.

\begin{table}[t]
\caption{Expert-normalized return on HumEnv across five official FB-CPR checkpoints.}
\label{tab:app_checkpoints}
\centering
\footnotesize
\setlength{\tabcolsep}{3pt}
\renewcommand{\arraystretch}{1.08}

\begin{tabular*}{0.92\linewidth}{
@{\extracolsep{\fill}} l c c c c c @{}
}
\toprule
Checkpoint & WR & $\Psi_G$ & OLS & BLS\at{OLS} & BLS\at{WR} \\
\midrule
S-1
& 0.677 & 0.673 & 0.509 & 0.686 & \textbf{0.744} \\

S-2
& 0.575 & 0.573 & 0.491 & 0.638 & \textbf{0.696} \\

S-3
& 0.575 & 0.560 & 0.480 & 0.641 & \textbf{0.696} \\

S-4
& 0.624 & 0.622 & 0.524 & 0.655 & \textbf{0.679} \\

S-5
& 0.615 & 0.640 & 0.477 & 0.622 & \textbf{0.686} \\

\midrule
Mean
& 0.613\pmstd{0.038}
& 0.614\pmstd{0.042}
& 0.496\pmstd{0.018}
& 0.648\pmstd{0.022}
& \textbf{0.700}\pmstd{0.023} \\
\bottomrule
\end{tabular*}
\end{table}

As shown in Table~\ref{tab:app_checkpoints}, \method{} improves both anchors across the five checkpoints. With OLS as the anchor, the average expert-normalized return increases from $0.496$ to $0.648$; with WR as the anchor, it increases from $0.613$ to $0.700$. This shows that the benefit of \method{} is robust across independently trained FB-CPR checkpoints rather than being specific to a particular pretrained representation.

\end{document}